\documentclass[twoside]{article}

\usepackage{algorithm}
\usepackage{algorithmic}
\renewcommand{\algorithmiccomment}[1]{\hfill$\triangleright$ #1}
\usepackage[scale=2]{ccicons}
\usepackage{pgfplots}
\usepgfplotslibrary{dateplot}
\usepackage{xspace}
\usepackage{graphicx}
\usepackage{booktabs}
\usepackage{tikz}
\usepackage{amsthm}
\usepackage{amsmath}
\usepackage{amsfonts}
\usepackage{amssymb}
\usepackage{bbm}
\usepackage{bm}
\usepackage{listings}
\usepackage{multirow}
\usepackage{subcaption}
\usepackage{comment}
\usetikzlibrary{shapes,arrows,calc,positioning}
\usetikzlibrary{graphs}
\usepackage{bm}
\usepackage{mathtools} 
\usepackage{caption}
\usepackage{pgfplots}
\pgfplotsset{compat=1.14}
\usepackage{wrapfig}
\usepackage{pdfpages}
\usepackage{longtable}
\usepackage{lscape}
\usepackage{rotating}
\usepackage{comment}
\usepackage{ed}
\usepackage{afterpage}
\usepackage{tabstackengine}
\TABstackText
\usepackage{xcolor}
\definecolor{Gray}{gray}{0.9}
\definecolor{Green}{RGB}{152,251,152}
\definecolor{DarkRed}{RGB}{192,0,0}
\definecolor{DarkGreen}{RGB}{84,130,53}
\definecolor{DarkBlue}{RGB}{68,114,196}
\definecolor{DarkOrange}{RGB}{255,192,0}
\usepackage[all]{nowidow}
\usepackage{babel}
\usepackage{tabularx,ragged2e}
\newcolumntype{C}{>{\Centering\arraybackslash}X} 
\usepackage{adjustbox}
\usepackage{setspace}
\usepackage[flushleft]{threeparttable}
\usepackage{hyperref}
\usepackage{stmaryrd}
\usepackage{amsthm}
\usepackage{bbm}
\usepackage{mdframed}
\usepackage{thmtools, thm-restate}
\usepackage{mathpazo}
\usepackage{tcolorbox}
\usepackage[round]{natbib}

\usepackage[accepted]{aistats2026}
\theoremstyle{plain}
\newtheorem{theorem}{Theorem}[section]

\theoremstyle{definition}

\theoremstyle{remark}

\DeclareMathOperator*{\argmin}{\arg\min}

\DeclareMathOperator*{\argmax}{\arg\max}

\newcommand{\calX}{\mathcal{X}}
\newcommand{\calY}{\mathcal{Y}}

\newcommand{\calO}{\mathcal{O}}

\newcommand{\calQ}{\mathcal{Q}}

\newcommand{\calT}{\mathcal{T}}

\newcommand{\calU}{\mathcal{U}}

\newcommand{\calST}{\mathcal{S}_{\calT}}
\newcommand{\calVT}{\mathcal{V}_{\mathcal{T}}}

\newcommand{\RT}{R_{\calT}}
\newcommand{\haty}{\hat{y}}
\newcommand{\hatp}{\hat{p}}

\newcommand{\sety}{\hat{Y}}
\newcommand{\setv}{\hat{V}}

\newcommand{\hatP}{\hat{P}}

\newcommand{\bbR}{\mathbb{R}}
\newcommand{\bbP}{\mathbb{P}}

\newcommand{\bbmI}{\mathbbm{1}}

\newcommand{\bx}{\boldsymbol{x}}

\newcommand{\powerset}{2^{\mathcal{Y}}}

\allowdisplaybreaks

\begin{document}

%

%

\twocolumn[

\aistatstitle{Conformal Prediction in Hierarchical Classification with Constrained Representation Complexity}

\aistatsauthor{Thomas Mortier$^{1,5}$ \And Alireza Javanmardi$^{2,3}$ \And  Yusuf Sale$^{2,3}$ \And Eyke H\"ullermeier$^{2,3,4}$ \And Willem Waegeman$^{5}$}
\vspace{0.5em}
\aistatsaddress{$^{1}$Department of Environment, \\Ghent University \And  
$^{2}$Institute of Informatics, \\LMU Munich \And 
$^{3}$MCML,\\ Munich}
\vspace{-1.5em}
\aistatsaddress{
$^{4}$DFKI (DSA), \\ Kaiserslautern \And
$^{5}$Dept. of Data Analysis and\\ Mathematical Modelling, Ghent University} ]

\begin{abstract}

Conformal prediction has emerged as a widely used framework for constructing valid prediction sets in classification and regression tasks. In this work, we extend the split conformal prediction framework to hierarchical classification, where prediction sets are commonly restricted to internal nodes of a predefined hierarchy, and propose two computationally efficient inference algorithms. The first algorithm returns internal nodes as prediction sets, while the second one relaxes this restriction. Using the notion of representation complexity, the latter yields smaller set sizes at the cost of a more general and combinatorial inference problem. Empirical evaluations on several benchmark datasets demonstrate the effectiveness of the proposed algorithms in achieving nominal coverage.

\end{abstract}


\section{Introduction}
\label{sec:intro}

In multi-class classification, a classifier can be uncertain about the predicted class label for a given test instance. In such cases, it can be beneficial to return set-valued predictions, i.e.\ sets of classes rather than individual classes. This is particularly relevant in hierarchical classification, where the class space is organised in a hierarchical structure, such as in medical diagnosis, where diseases are organised in a tree structure based on the International Classification of Diseases (ICD)~\citep{who78icd9}. 

In hierarchical classification, set-valued predictions are often restricted to internal nodes of the hierarchy. Such predictions have a clear semantic interpretation and can be computed using efficient inference algorithms~\citep{freitas07hc,bi15bohmlc,rangwala17lhc,yang17cautioushmc,wang21hierarchical,valmadre22hierarchical}. However, a restriction of this kind can affect efficiency when the classifier is uncertain between classes in different branches of the hierarchy. In these cases, predicting a single internal node typically yields large and uninformative sets. 

To address this, some approaches allow any subset of classes as predictions, which improves flexibility but comes at the cost of higher semantic complexity and reduced interpretability~\citep{oh17topkhc}. More recently, a set-based utility maximisation framework has been proposed that makes a compromise between the two aforementioned extremes by introducing the notion of representation complexity~\citep{mortier22svprc}. The assumption behind this notion is that inner nodes of a hierarchy have semantic meaning and can therefore be used as meaningful stand-alone predictions. By limiting the number of internal nodes used to represent a prediction, the proposed framework allows users to control the trade-off between efficiency and interpretability. 

\begin{figure}[t]
    \centering
    \includegraphics[width=0.9\columnwidth,height=0.4\columnwidth]{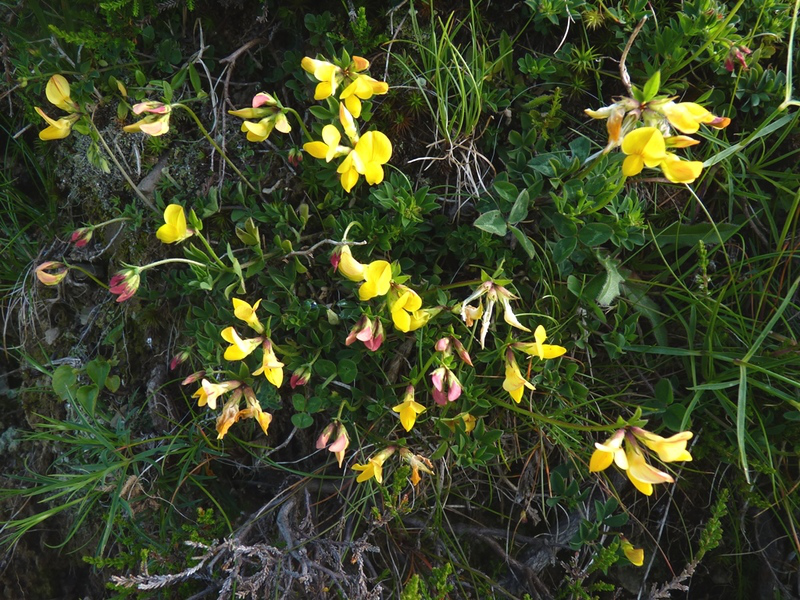}
    \caption{A sample image of the \emph{Lotus corniculatus} species from the PlantCLEF 2015 dataset~\citep{goeau15lifeclef}.}
    \label{fig:ill}
\end{figure}

To illustrate the usefulness of representation complexity, consider the PlantCLEF 2015 dataset~\citep{goeau15lifeclef}. This dataset consists of over one hundred thousand images representing 1,000 species of trees, herbs, and ferns native to the Western European region and is characterised by significant class ambiguity, making accurate predictions on species level often impossible. In Figure~\ref{fig:ill}, an example image is shown corresponding to the \emph{Lotus corniculatus} species. Using our proposed algorithms, two set-valued predictions were computed, with a restricted representation complexity of $r=1$ and $r\leq3$, respectively.
The former yields the root node of the hierarchy, i.e.\ all 1,000 species, and is completely uninformative. Increasing the representation complexity to three improves efficiency and yields the prediction $\{\emph{Lotus corniculatus},\emph{Tulipa sylvestris}, \emph{Ficaria verna}\}$, i.e., three visually-related species including the true label.

\noindent\textbf{Contributions.} In this work, we build further upon the concept of representation complexity and extend the split conformal prediction framework to the hierarchical multi-class classification setting. Conformal prediction provides valid and efficient set-valued predictions for any learning algorithm, without making any assumptions on the underlying data distribution~\citep{vovk05algorithmic}. In particular, given a trained (hierarchical) classifier, a desired coverage level $1-\alpha\in(0,1)$, calibration samples $\{(\bx_{i},y_{i})\}_{i=1}^{N} \subset \mathcal{X} \times \calY$ and a test sample $(\bx_{N+1},y_{N+1})$, drawn i.i.d.\ from an unknown distribution $P$, we would like to construct a prediction set $\sety\in\powerset$ with representation complexity $r$ for the test instance $(\bx_{N+1},y_{N+1})$, such that the following marginal validity guarantee holds:
\begin{equation}
  \label{eq:problem}
\bbP\left[y_{N+1}\in\sety(\bx_{N+1})\right] \geq 1-\alpha\,,\,\,\textrm{s.t.}\,\, \RT(\sety)\leq r\,,
\end{equation}
where $\RT(\sety)$ denotes the representation complexity of the set $\sety$, given some tree structure $\calT$. The probability is taken over all $N+1$ samples and we require that (\ref{eq:problem}) holds for any fixed $\alpha, N, r$ and $P$. Figure~\ref{fig:hierarchy:ex}, which will be explained in detail later on, intuitively illustrates the concept of representation complexity. In this example of a hierarchical tree structure with eight classes, the set $\sety = \{1,2,4,7,8\}$ has a representation complexity of three, as it requires three nodes to represent this set: $v_4$, $v_7$ and $v_{11}$.

After reviewing some essentials on hierarchical probabilistic classification and conformal prediction in Section~\ref{sec:backrw}, we propose in Section~\ref{sec:methods} two algorithms that construct valid prediction sets, in the sense of satisfying (\ref{eq:problem}). The first algorithm is designed for the restrictive case $r=1$, while the second extends to $r>1$. Moreover, we show that our algorithms possess distribution-free finite sample guarantees. Finally, in Section~\ref{sec:results}, we evaluate the proposed algorithms in terms of coverage and efficiency on a range of benchmark datasets.

\begin{figure*}[t]
  \centering
  \scalebox{0.8}{
  \begin{tikzpicture}[sibling distance=12em,
      every node/.style = {align=center},
      line/.style={draw, -latex'},
      edge from parent/.style={draw,-latex',font=\normalsize},
      level 1/.style={sibling distance=80mm,font=\normalsize}, 
      level 2/.style={sibling distance=40mm,font=\normalsize},
      level 3/.style={sibling distance=20mm,text width=2cm,font=\Large}] 
      \node[font=\normalsize] {$v_1=\{1,2,3,4,5,6,7,8\}$}
        child { node (ch1){$v_2=\{1,2,3,4\}$} 
          child { node (ch3){$v_4=\{1,2\}$} 
            child { node {$\substack{v_8=\{1\},\\ 0.15}$} }
            child { node {$\substack{v_9=\{2\},\\0.13}$} } }
          child { node (ch4){$v_5=\{3,4\}$} 
            child { node {$\substack{v_{10}=\{3\},\\ 0.08}$} }
            child { node {$\substack{v_{11}=\{4\},\\0.125}$} } } }
        child { node (ch2){$v_3=\{5,6,7,8\}$} 
          child { node (ch5){$v_6=\{5,6\}$} 
            child { node {$\substack{v_{12}=\{5\},\\0.14}$} }
            child { node {$\substack{v_{13}=\{6\},\\0.125}$} } }
          child { node (ch6){$v_7=\{7,8\}$} 
            child { node {$\substack{v_{14}=\{7\},\\0.125}$} }
            child { node {$\substack{v_{15}=\{8\},\\0.125}$} } } };
  \end{tikzpicture}
  }
  \caption{An example of a tree structure\index{tree structure} $\calT$ with class space $\calY=\{1,\ldots,8\}$ and nodes $\calVT=\{v_{1},\ldots,v_{15}\}$. The root $v_{1}$ represents the class space $\calY$ and leaves $\{v_{8},\ldots,v_{15}\}$ represent the individual classes. The numbers in the leaf nodes represent the class probabilities for an instance $\bx$. }
  \label{fig:hierarchy:ex}
\end{figure*}

\section{Background, related work and formal description of existing concepts}
\label{sec:backrw}

\subsection{Hierarchical multi-class classification}
\label{sec:hierclass}

In a standard multi-class classification setting, one assumes that training and test data are i.i.d.\ according to an unknown distribution $P(\bx,y)$ on $\calX\times\calY$, with $\calX$ some instance space (e.g.\ images, documents, etc.) and $\calY=\{c_1,\ldots,c_K\}$ a class space consisting of $K$ classes. Probabilistic multi-class classifiers estimate the conditional class probabilities $P(\cdot \,|\,\bx)$ over $\calY$, such that $0 \leq P(c\,|\,\bx) \leq 1$ for all $c \in \calY$ and  $\sum_{c \in \calY} P(c\,|\, \bx) = 1 \,.$ This distribution can be estimated using a wide range of well-known probabilistic models, such as logistic regression, linear discriminant analysis, gradient boosting trees or neural networks with a softmax output layer. At prediction time, set-valued prediction algorithms, such as split conformal prediction, return sets $\sety$ that are subsets of $\calY$. The probability mass of such a set can be computed as $P(\sety \,|\, \bm{x}) = \sum_{c \in \sety} P(c \,|\, \bm{x})$.

In this paper, we will consider a hierarchical multi-class classification setting. Hence, we assume that a domain expert has defined a hierarchy over the class space, in the form of a tree structure $\calT$ that in general contains $M$ nodes. $\calVT=\{v_1,\ldots,v_M\}$ will denote the set of nodes and every node identifies a set of classes. As special cases, the root $v_1$ represents the class space $\calY$, and the leaves represent individual classes; see Figure~\ref{fig:hierarchy:ex} for a simple example. In hierarchical classification, the strong restriction $\sety \in \calVT$ is typically made for predicted sets; see e.g.\ \citet{bi15bohmlc}. The probability mass  of such a set can be computed as $P(v\,|\,\bx) = \sum_{c \in v} P(c\,|\,\bx)$, or by using the chain rule of probability:
\begin{equation}
\label{eq:hierfac}
P(v\,|\,\bx)  = \prod_{v' \in \mathrm{Path}(v)} P(v' \,|\, \mathrm{pa}(v'), \bx) \,,
\end{equation}
where $\mathrm{Path}(v)$ is a set of nodes on the path connecting the node $v$ and the root of the tree structure. $\mathrm{pa}(v)$ gives the parent of node $v$ and $P(v \,|\, \mathrm{pa}(v), \bx)$ represents the branch probability of node $v$ given its parent $\mathrm{pa}(v)$. Note that for the root node $v_1$ one has $P(v_1 \,|\, \mathrm{pa}(v_1), \bx) = 1$. In Figure~\ref{fig:hierarchy:ex}, the branch probabilities of the root node $v_1$ are given by $P(v_2 \,|\, v_1, \bx) = 0.485$ and $P(v_3 \,|\, v_1, \bx) = 0.515$. In order to estimate the branch probabilities, one can train any multi-class probabilistic classifier in each internal node of the tree. Classical models of that kind include nested dichotomies~\citep{fox97reg,frank04nested,melnikov18nested}, conditional probability estimation trees~\citep{beygelzimer09cpte} and probabilistic classifier trees~\citep{dembczynski16consistencyop}. In neural networks with a hierarchical softmax output layer, all nodes are trained simultaneously~\citep{morin05hiers}. 

\subsection{Inference in hierarchical classification}
In this work, we do not focus on the training algorithms. Instead, we assume that a probabilistic model $\hatP$ has been estimated, either with classical models or using a hierarchical factorization as in (\ref{eq:hierfac}), and we focus on the prediction task. In particular, we would like to predict sets $\sety\in\powerset$ that satisfy (\ref{eq:problem}), with a restriction on the representation complexity of the predicted set~\citep{mortier22svprc}. The representation complexity is defined as the minimal number of nodes needed to represent the set $\sety$ in the tree structure $\calT$. More formally, let $\calST(\sety)$ denote the set of all disjoint combinations of tree nodes that represent the set $\sety$:
\begin{equation*}
	\label{eq:s}
	\calST(\sety) = \left\{\setv \subset \calVT: \bigcup_{v_{i}\in \setv} v_{i}=\sety \land \bigcap_{v_{i}\in\setv} v_{i}=\emptyset \right\}\, .
\end{equation*}
Then, we define the representation complexity of the prediction $\sety$ as
\begin{align}
    \label{eq:c}
	\RT(\sety)=\min_{\setv\in \calST(\sety)} |\setv|\,,
\end{align}
with $|\setv|$ the cardinality of $\setv$. For example, the set $\sety = \{1,2,4,7,8\}$ of the hierarchy shown in Figure~\ref{fig:hierarchy:ex} has representation complexity three, because three nodes are needed to represent this set: $v_4$, $v_7$ and $v_{11}$. 


\subsection{Randomized nested prediction sets with split conformal prediction}
\label{sec:ncp}

Now we will review a general procedure for valid set-valued predictions in flat classification (i.e.\ ignoring the hierarchy), following the work of~\citet{gupta22nested,romano20classification,angelopoulos20raps}. Compared to traditional conformal prediction---which starts from the notion of a nonconformity score---this procedure departs from a sequence of nested set-valued predictions, where the set size depends on a threshold $\tau$. Furthermore, an independent calibration set is used to tune the threshold $\tau$ such that $(1-\alpha)$-coverage is guaranteed on future test samples. The use of a single calibration set is better known as split conformal prediction and gives rise to computationally efficient valid set-valued predictions~\citep{papadopoulos02inductive,lei18distribution}.

More formally, let $\{(\bx_{i},y_{i})\}_{i=1}^{N}$ be an i.i.d.\ sequence of $N$ samples from the unknown distribution $P$, and assume that these samples were not used for model training. Let $\sety(\bx,u,\tau):\calX \times [0,1] \times \bbR \rightarrow \powerset$ be a set-valued predictor. In the spirit of \citet{romano20classification,angelopoulos20raps}, the second argument $u$ represents a random draw from a uniform distribution $\calU(0,1)$ and is included to allow for randomized prediction sets. The third argument $\tau$ is a threshold that controls the size of the predicted set. Assume that the sets are indexed by $\tau$ and are nested:
\begin{equation}
  \label{eq:nested}
  \sety(\bx,u,\tau_{1}) \subseteq \sety(\bx,u,\tau_{2})\qquad \text{if}\quad \tau_{1} \leq \tau_{2}\,.
\end{equation}
In order to find the optimal threshold $\tau^{*}$ that guarantees $(1-\alpha)$-coverage, one needs to find the smallest $\tau$ such that the predicted set $\sety(\bx,u,\tau)$ contains at least $\lceil (N+1)(1-\alpha)\rceil$ samples:
\begin{multline}
  \label{eq:nested:tau}
  \tau^{*} = \inf \big\{ \tau \in [0,1]: |\{i: y_{i}\in \sety(\bx_{i},u_{i},\tau)\}| \\ \geq \lceil (1-\alpha)(N+1)\rceil\big\} \,.  
\end{multline}
The set-valued predictor $\sety(\bx_{N+1},u_{N+1},\tau^{*})$, with $\tau^{*}$ in (\ref{eq:nested:tau}), has marginal validity guarantees, as shown by the theorem below.

\begin{theorem}[Marginal validity of nested conformal prediction~\citep{angelopoulos20raps}]
  Assume an exchangeable sequence $\{(\bx_{i},y_{i},u_{i})\}_{i=1}^{N+1}$ and let $\sety(\bx,u,\tau)$ be a set-valued predictor that satisfies (\ref{eq:nested}). Furthermore, assume that $\exists \tau \in \bbR: \sety(\bx,u,\tau)=\calY$. Then, for $\tau^{*}$ in (\ref{eq:nested:tau}) and any $\alpha \in (0,1)$, the following marginal coverage guarantee holds:
  $$ 1 - \alpha \leq \bbP\big[y_{N+1}\in\sety(\bx_{n+1},u_{N+1},\tau^{*})\big]\,. $$
\end{theorem}

In the literature, several conformal prediction methods can be found that satisfy (\ref{eq:nested}). For example, \citet{romano20classification} propose the following set-valued predictor, called adaptive prediction sets (APS):
\begin{align}
  \label{eq:aps}
  \sety_{APS}(\bx,u,\tau) = \big\{y\in \calY: \hat{\rho}(y;\bx) +
  u\cdot \hatP(y\,|\,\bx)\leq \tau \big\}\,,
\end{align}
with $\hat{\rho}(y;\bx)=\sum_{y'\in\calY}\hatP(y'\,|\,\bx)\bbmI_{\hatP(y'\,|\,\bx)>\hatP(y\,|\,\bx)}$ the probability mass of the labels more likely than $y$. At the heart of their method is the randomization term $u\cdot \hatP(y\,|\,\bx)$, which is used to achieve exact nominal coverage. In addition, by means of the cumulative distribution $\hat{\rho}(y;\bx)$, this method allows to adapt effectively to complex data distributions, achieving a better conditional coverage compared to alternative approaches that rely on the mode of the distribution, such as the least ambiguous classifier proposed by \citet{sadinle19least}. \citet{angelopoulos20raps} propose a variant of the above method, called the regularized adaptive prediction sets (RAPS) method. By introducing a regularization term in (\ref{eq:aps}), 
small set sizes are encouraged, especially when facing noisy probability estimates for classes with low probability.

%

Alternative conformal prediction methods have been introduced, focusing on minimizing set size~\citep{sadinle19least,gao2025volume} and achieving approximate coverage across the entire feature space~\citep{foygel21limits,cauchois21knowing,romano19conformalized}. \citet{goren24hierarchical} and \citet{angelopoulos23conformalriskcontrol} propose conformal prediction methods that control the risk with respect to bounded non-increasing loss functions. Interestingly, these methods are also applicable to the setting of hierarchical classification, e.g.\ through the use of the tree-distance loss, introduced by \citet{bi15bohmlc} as a way of evaluating set-valued predictions in multi-label classification. However, set-valued predictions of that kind lack meaningful and practical interpretation, compared to methods that rely on the traditional notion of coverage. Furthermore, these methods are not directly applicable to our work, as they do not allow the incorporation of representation complexity, given the properties of the loss functions above. 

\citet{gao2025volume} recently proposed a novel conformal predictor for regression, by implementing a dynamic programming algorithm that finds a union of $k$-intervals on the real line. When the output distribution is multi-modal, a union of $k$-intervals can be more informative and efficient compared to a single prediction interval (as obtained with most regression methods). Our work has a similar motivation, but adopted to hierarchical classification: when the classifier is uncertain between classes in different branches of the hierarchy, restricting the representation complexity to one is overly restrictive, while a higher complexity yields more efficient predictions.
 
\section{Proposed methods}
\label{sec:methods}

Building upon the concepts in Section~\ref{sec:ncp}, we are now ready to discuss two algorithms that provide valid set-valued predictions for hierarchical classification that satisfy (\ref{eq:problem}) above. The first algorithm provides marginal validity guarantees in classical hierarchical classification settings, i.e.\ where sets are restricted to nodes of the tree structure, thus $\RT(\sety)=1$. The second algorithm provides marginal validity guarantees in restricted set-valued prediction settings, i.e.\ where the representation complexity of the predicted set is restricted to $\RT(\sety)\leq r$ for a user-defined value $r$. Note that traditional conformal prediction in flat classification is obtained for the second algorithm when there is no restriction on the representation complexity (e.g.\ $r=K$).

\subsection{Conformal restricted set-valued prediction}
\label{sec:crsvp}

\begin{figure}[t]
\begin{minipage}[t]{0.48\textwidth}
\begin{algorithm}[H]
  \begin{small}
  \caption{CRSVP calibration -- \textbf{Input:} $\{(\bx_{i},y_{i},u_{i})\}_{i=1}^{N}, \hatP, \calVT$, \textbf{Output:} Threshold in (\ref{eq:nested:tau}).}
  \begin{algorithmic}[1]
    \FOR{$i=1,\ldots,N$}
      \STATE $\sety \gets \argmax_{c \in \calY} \hatP(c\,|\,\bx_{i})$ 
      \STATE $ \hatp_{\sety} \gets \hatP(\sety\,|\,\bx_{i}), \hatp_{\sety'} \gets 0$
      \WHILE{$y_{i} \notin \sety$}
          \STATE $\hatp_{\sety'} \gets \hatp_{\sety}$, $\hatp_{\sety} \gets \hatP(\mathrm{pa}(\sety)\,|\,\bx_{i})$
          \STATE $\sety' \gets \sety$, $\sety \gets \mathrm{pa}(\sety)$
      \ENDWHILE 
      \STATE $\tau_{i} \gets \hatp_{\sety}-u_{i}\cdot(\hatp_{\sety}-\hatp_{\sety'})$
    \ENDFOR
    \STATE $\tau^{*} \gets$ the $\lceil (1-\alpha)(N+1)\rceil$-th largest value in $\{\tau_{i}\}_{i=1}^{N}$ 
    \STATE \textbf{return} $\tau^{*}$ 
  \end{algorithmic}
  \label{alg:crsvp:cal}
  \end{small}
\end{algorithm}
\end{minipage}
\hfill
\begin{minipage}[t]{0.48\textwidth}
\begin{algorithm}[H]
    \begin{small}
    \caption{CRSVP inference -- \textbf{Input:} $\bx, \tau^{*}, u, \hatP, \calVT$, \textbf{Output:} Set-valued prediction in (\ref{eq:crsvp}).}
    \begin{algorithmic}[1]
        \STATE $\sety \gets \argmax_{c \in \calY} \hatP(c\,|\,\bx), \sety'\gets \emptyset$
        \STATE $ \hatp_{\sety} \gets \hatP(\sety\,|\,\bx), \hatp_{\sety'} \gets 0$
        \WHILE{$\mathrm{pa}(\sety') \neq \emptyset$}
          \IF{$\hatp_{\sety} \geq \tau^{*}$}
            \STATE \textbf{break}
          \ENDIF
          \STATE $\hatp_{\sety'} \gets \hatp_{\sety}$, $\hatp_{\sety} \gets \hatP(\mathrm{pa}(\sety)\,|\,\bx)$
          \STATE $\sety' \gets \sety$, $\sety \gets \mathrm{pa}(\sety)$
        \ENDWHILE
        \IF{$\hatp_{\sety}-u\cdot(\hatp_{\sety}-\hatp_{\sety'}) \geq \tau^{*}$}
            \STATE \textbf{return} $\sety'$
        \ELSE
            \STATE \textbf{return} $\sety$
        \ENDIF
    \end{algorithmic}
    \label{alg:crsvp}
    \end{small}
\end{algorithm}
\end{minipage}
\end{figure}

The first approach, called conformal restricted set-valued prediction (CRSVP), predicts sets that are restricted to nodes of the hierarchy, thus having representation complexity $\RT(\sety)=1$. Assume one has fitted a classifier using a training dataset as described in Section~\ref{sec:hierclass}, and let $\haty(\bx)$ denote the mode of the distribution $\hatP(.\,|\,\bx)$, i.e.\ the leaf node with highest probability mass. For our first approach, we consider the following set-valued predictor: 
\begin{multline}
  \label{eq:crsvp}
  \sety_{1}(\bx,u,\tau) = \argmax_{\sety \in \mathrm{Path}(\haty(\bx))} \big\{|\sety|:\hatP(\sety\,|\,\bx)+ \\ u\cdot\hatP(\mathrm{pa}(\sety)\setminus \sety\,|\,\bx) \leq \tau \big\}\,,
\end{multline}
with $\mathrm{Path}(v)$ the path to the root, and $\mathrm{pa}(v)$ the parent of a node $v$, as defined in Section~2. $\hatP(\sety\,|\,\bx)$ can be computed using (\ref{eq:hierfac}), but also observe that the probability mass of an internal node $v$ corresponds to $\hatP(v\,|\,\bx) = \sum_{c \in v} \hatP(c \,|\, \bx)$, i.e.\ the sum of the probability masses of the leaf nodes that are descendants of $v$.  However, computing the probability mass of an internal node and its associated set of classes in this way is computationally less efficient than using the chain rule in (\ref{eq:hierfac}), when hierarchical classifiers have been fitted during training.

It is clear that the set-valued predictor in (\ref{eq:crsvp}) satisfies the nestedness property in (\ref{eq:nested}). Indeed, starting from the mode of the distribution $\haty(\bx)$, every node on the path connecting that node and the root is nested by definition of the hierarchical tree structure $\calT$. In line with \citet{angelopoulos20raps}, the first term increases as we move up the tree, while the second term contains a random draw from the uniform distribution $\mathcal{U}(0,1)$,  for handling discrete jumps in probability mass when following the path towards the root. The randomization is needed to prevent over-coverage. 

The algorithm for calculating the threshold in (\ref{eq:nested:tau}) using the calibration dataset is presented in Algorithm~\ref{alg:crsvp:cal}. This algorithm first computes the mode of the estimated class probabilities for each instance in the calibration dataset. For small hierarchies, this can be done using exhaustive search, but more efficient inference algorithms have been developed when the computational cost becomes a burden~\citep{dembczynski12chaining,kumar13beam,mena15astar}. Subsequently, starting from the mode, Algorithm~\ref{alg:crsvp:cal} computes for every calibration instance the internal node, on the path to the root, that also includes the true class label. Nonconformity scores that incorporate a randomization component are constructed for these internal nodes, and the final threshold is deduced from the scores. Algorithm~\ref{alg:crsvp:cal} has a worst case $\calO(N \, K)$ time complexity when exhaustive search is applied in the first step. 

The pseudocode for computing the set-valued prediction for a new test instance, as defined in (\ref{eq:crsvp}), is shown in Algorithm~\ref{alg:crsvp}. Given a hierarchical classifier, the computational complexity during test time is given by $\calO(\log K)$.

\subsection{Conformal set-valued prediction with representation complexity}
\label{sec:crsvpr}

\begin{figure}[t!]
  \begin{minipage}[t]{0.48\textwidth}
\begin{algorithm}[H]
  \begin{small}
  \caption{CRSVP-$r$ calibration -- \textbf{Input:} $\{(\bx_{i},y_{i},u_{i})\}_{i=1}^{N}, r, \hatP, \calVT$, \textbf{Output:} Threshold in (\ref{eq:nested:tau}).}
  \begin{algorithmic}[1]
    \FOR{$i=1,\ldots,N$}
      \STATE $k \gets 1$
      \STATE $\sety^{(k-1)} \gets \emptyset, \sety^{(k)} \gets \argmax_{c \in \calY} \hatP(c\,|\,\bx_{i})$
      \STATE $\hatp_{\sety^{(k-1)}} \gets 0, \hatp_{\sety^{(k)}} \gets \hatP(\sety^{(k)}\,|\,\bx_{i})$
      \WHILE{$k < K$}
        \IF{$y_{i} \notin \sety^{(k)}$}
          \STATE $\sety^{(k)'} \gets A_{r}(S_{k};\bx_{i})$ (by means of Algorithm~\ref{alg:dpmincovset})
          \STATE $\hatp_{\sety^{(k)'}} \gets \hatP(\sety^{(k)'}\,|\,\bx_{i})$
          \IF{$|\sety^{(k)'}| \neq |\sety^{(k-1)}|$}
            \STATE $\sety^{(k)}, \hatp_{\sety^{(k)}} \gets \sety^{(k)'}, \hatp_{\sety^{(k)'}}$
          \ENDIF
          \STATE $k \gets k+1$
        \ELSE
          \STATE \textbf{break}
        \ENDIF
      \ENDWHILE
      \STATE $\tau_{i} \gets \hatp_{\sety^{(k)}}-u_{i}\cdot(\hatp_{\sety^{(k)}}-\hatp_{\sety^{(k-1)}})$
    \ENDFOR
    \STATE $\tau^{*} \gets$ the $\lceil (1-\alpha)(N+1)\rceil$-th largest value in $\{\tau_{i}\}_{i=1}^{N}$ 
    \STATE \textbf{return} $\tau^{*}$ 
  \end{algorithmic}
  \label{alg:crsvpr:cal}
  \end{small}
\end{algorithm}
\end{minipage}
\hfill
\begin{minipage}[t]{0.48\textwidth}
\begin{algorithm}[H]
  \begin{small}
  \caption{CRSVP-$r$ inference -- \textbf{Input:} $\bx, \tau^{*}, u, r, \hatP, \calVT$, \textbf{Output:} Set-valued prediction in (\ref{eq:crsvpr}).}
  \begin{algorithmic}[1]
    \STATE $k \gets 1$
    \STATE $\sety^{(k-1)} \gets \emptyset, \sety^{(k)} \gets \argmax_{c \in \calY} \hatP(c\,|\,\bx)$
    \STATE $\hatp_{\sety^{(k-1)}} \gets 0, \hatp_{\sety^{(k)}} \gets \hatP(\sety^{(k)}\,|\,\bx)$
    \WHILE{$k < K$}
      \IF{$y^{(k)} \notin \sety^{(k-1)}$}
        \STATE $\sety^{(k)'} \gets A_{r}(S_{k};\bx)$ (by means of Algorithm~\ref{alg:dpmincovset})
        \STATE $\hatp_{\sety^{(k)'}} \gets \hatP(\sety^{(k)'}\,|\,\bx)$
        \IF{$|\sety^{(k)'}| \neq |\sety^{(k-1)}|$}
          \STATE $\sety^{(k)}, \hatp_{\sety^{(k)}} \gets \sety^{(k)'}, \hatp_{\sety^{(k)'}}$
        \ENDIF
        \IF{$\hatp_{\sety^{(k)}} \geq \tau^{*}$}
          \STATE \textbf{break}
        \ENDIF
        \STATE $k \gets k+1$
      \ENDIF
    \ENDWHILE
    \IF{$\hatp_{\sety^{(k)}}-u\cdot(\hatp_{\sety^{(k)}}-\hatp_{\sety^{(k-1)}}) \geq \tau^{*}$}
        \STATE \textbf{return} $\sety^{(k-1)}$
    \ELSE
        \STATE \textbf{return} $\sety^{(k)}$
    \ENDIF
  \end{algorithmic}
  \label{alg:crsvpr}
  \end{small}
\end{algorithm} 
\end{minipage}
\end{figure}
\begin{figure}[t!]
\begin{minipage}[t]{0.48\textwidth}
\begin{algorithm}[H]
    \begin{small}
    \caption{Dynamic programming solution for set of lowest common ancestors -- 
    \textbf{Input:} $\bx, S_{k}, r, \hatP, \calVT, A_{r}(S_{k-1};\bx)$, 
    \textbf{Output:} $A_{r}(S_{k};\bx)$}
    \begin{algorithmic}[1]
    \STATE $\calQ \leftarrow \emptyset$ \COMMENT{initialize a queue for visiting nodes}
    \FOR{all $v \in \calVT$} 
        \STATE $s^{1}(v) \leftarrow \emptyset, \ldots, s^{r}(v)\leftarrow \emptyset$
        \IF{$|v|=1 \;\wedge\; v \cap S_{k} \neq \emptyset$} 
            \IF{$\mathrm{pa}(v) \notin \calQ$}
                \STATE $\calQ\mathrm{.add}(\mathrm{pa}(v))$
            \ENDIF
            \STATE $s^{1}(v) \leftarrow v, \ldots, s^{r}(v)\leftarrow v$ \COMMENT{initialize leaves that overlap with $S_{k}$}
        \ENDIF
    \ENDFOR
    \WHILE{$\calQ$ is not empty}
        \STATE $v \leftarrow \calQ$.pop()
        \STATE $T \leftarrow \{v': v'\in \mathrm{ch}(v) \wedge  v' \cap S_{k} \neq \emptyset\}$ \COMMENT{visit all children that overlap with $S_{k}$}
        \FOR{$r_{i}=1$ to $r$}
            \IF{$|T| > r_{i}$}
                \STATE $s^{r_{i}}(v) \leftarrow v$
            \ELSE
                \STATE $M \leftarrow$ all compositions of $r_{i}$ into $|T|$ elements
                \STATE $u^{*} \leftarrow +\infty$
                \FOR{$(i_{0},\ldots,i_{|T|-1}) \in M$}
                    \STATE $s \leftarrow s^{i_{0}}(T[0]) \cup \cdots \cup s^{i_{|T|-1}}(T[|T|-1])$
\IF{$|s|-\hatp_{s}<u^{*}$}
  \IF{$v$ is the root}
    \IF{$A_{r}(S_{k-1};\bx)\subseteq s$}
      \STATE $s^{r_{i}}(v) \leftarrow s$
      \STATE $u^{*} \leftarrow |s|-\hatp_{s}$
    \ENDIF
  \ELSE
    \STATE $s^{r_{i}}(v) \leftarrow s$
    \STATE $u^{*} \leftarrow |s|-\hatp_{s}$
  \ENDIF
\ENDIF
                \ENDFOR
            \ENDIF
        \ENDFOR
        \IF{$\mathrm{pa}(v) \notin \calQ \wedge \mathrm{pa}(v)\neq\emptyset$}
            \STATE $\calQ\mathrm{.add}(\mathrm{pa}(v))$
        \ENDIF
    \ENDWHILE
    \STATE \textbf{return} $s^{r}(v_{1})$
    \end{algorithmic}
    \label{alg:dpmincovset}
    \end{small}
\end{algorithm}
\end{minipage}

\end{figure}

It should be obvious that sets of representation complexity one quickly become very big, since they correspond to internal nodes of the hierarchy that exhibit a predefined coverage level. Therefore, we present a second approach, called conformal set-valued prediction with representation complexity (CRSVP-$r$), that relaxes the representation complexity constraint to $\RT(\sety) \leq r$ with $r$ a user-defined parameter. Let $S_{k} \equiv \{y^{(1)},\ldots,y^{(k)}\}$ denote the top-$k$ classes for $\bx$ when sorting the classes according to $\hatP(y\,|\,\bx)$, similar as in (\ref{eq:aps}). We define the following sequence of optimization problems for any $k \in \{1,...,K\}$:
\begin{equation}
  \label{eq:mincovset}
  A_{r}(S_{k};\bx) = \argmin_{\substack{\sety \in 2^{\calY}:R_{\calT}(\sety)\leq r,\\ A_{r}(S_{k-1};\bx)\,\cup\,y^{(k)} \subseteq \sety}} |\sety| - \hatP(\sety\,|\,\bx)\,,
\end{equation}
with $A_{r}(S_{1};\bx)=y^{(1)}$. Remark that these optimization problems have to be solved in the correct order, because of the nestedness condition. For fixed $k$ we refer to this optimization problem as \emph{finding the set of common ancestors}. It can be thought of as a variant of the well-known lowest common ancestor problem~\citep{aho73finding,harel84fast,bender00lca,dash13scalable}, i.e.\ instead of considering a single common ancestor of a set of leaves in a tree, we are looking for a set of $r$ non-overlapping ancestors whose descendants form a set of representation complexity $r$. Moreover, we introduce the nested set condition, and the second term in the objective function is also not considered in the lowest common ancestor problem. Those two changes are needed to make sure that (1) the solutions for increasing $k$ are nested, and (2) the solution for fixed $k$ is unique. Since probabilities are bounded between zero and one, the second term only plays a role for sets that have the same cardinality (i.e.\ the two terms are lexicographically ordered).    
As an example, consider the hierarchy depicted in Figure~\ref{fig:hierarchy:ex}. For $S_{3}=\{1,2,5\}$, the lowest common ancestor is given by $A_{1}(S_{3};\bx)=\{1,2,3,4,5,6,7,8\}$. The lowest set of common ancestors with representation complexity two is $A_{2}(S_{3};\bx)=\{1,2,5\}$, illustrating that set size can be drastically reduced by increasing the representation complexity.

Let $\sety^{(1)},\ldots,\sety^{(K)}$ denote the ordered sequence of unique lowest common ancestors $A_{r}(S_{1};\bx),\ldots,A_{r}(S_{K};\bx)$. We use this sequence to introduce the following set-valued predictor: 
\begin{multline}
  \label{eq:crsvpr}
  \sety_{\leq r}(\bx,u,\tau) = \argmax_{\sety^{(k)}\in\{\sety^{(1)},\ldots,\sety^{(K)}\}} \big\{|\sety^{(k)}|: \\ \hatP(\sety^{(k-1)}\,|\,\bx)+u\cdot\hatP(\sety^{(k)}\setminus\sety^{(k-1)}\,|\,\bx) \leq \tau \big\}\,.
\end{multline}

Due to the nested set constraint in (\ref{eq:mincovset}), the set-valued predictor in (\ref{eq:crsvpr}) satisfies the nestedness property in (\ref{eq:nested}). Thus, this set-valued predictor automatically has the classical marginal coverage guarantees of conformal prediction. Similar as for the first approach, the randomization term $u\cdot\hatP(\sety^{(k)}\setminus\sety^{(k-1)}\,|\,\bx)$ is again included to prevent over-coverage.

Let us now discuss the algorithmic aspects of computing the set of lowest common ancestors in (\ref{eq:mincovset}). At first impression, the combinatorial optimization problem in (\ref{eq:mincovset}) looks challenging, but it can be solved in an efficient manner with a divide-and-conquer algorithm that is described in Algorithm~\ref{alg:dpmincovset}.

For every internal node of the hierarchy, the optimization problem can be broken down into simpler subproblems of the same type. As an example, consider in a predefined hierarchy an internal node $v'$ that has four children $(v'_1,v'_2, v'_3,v'_4)$  containing at least one element of $S_k$, and assume that we aim to find the set of lowest common ancestors with representation complexity three for a set of classes $S_{k}$ that are all descendants of $v'$. As potential divisions of the representation complexity, one could find one common ancestor as descendants of children $v'_1$, $v'_3$, and $v'_4$, or as descendants of children $v'_1$, $v'_2$, and $v'_4$, or as two descendants of $v'_1$ combined with one descendant of $v'_4$, etc. In fact, many combinations are possible. Speaking in formal combinatorial terminology, one has to consider for the node $v$ all \emph{compositions} of the integer three into four elements. For each of these compositions, the optimization problem can be recursively divided into smaller problems with lower representation complexity, where the node $v$ is replaced by each of its non-zero children in the composition. However, implementing such a strategy in a recursive manner would heavily blow up the computations, since each of the smaller subproblems would need to be solved many times. 

A dynamic programming implementation, which avoids recursion by solving the smaller subproblems in a bottom-up manner, before tackling the more challenging optimization problems higher up in the hierarchy, is able to solve (\ref{eq:mincovset}) in an efficient manner for small values of $r$. The pseudocode of such an implementation is described in Algorithm~\ref{alg:dpmincovset}. For every internal node, $r$ local optimization problems are solved by varying the local representation complexity $r_i$ from one to $r$, and the solutions of these optimization problems are stored. By visiting children before their parents get analysed, one can guarantee that all needed quantities are known when the different compositions in an internal node need to be investigated. The most critical step in Algorithm~\ref{alg:dpmincovset} is line 20, where all compositions of the current representation complexity $r_i$ into $|T|$ elements are considered (with $T$ the set of all children that contain at least one element of $S_k$). Strictly speaking, this step has a runtime that is exponential in $r$. However, in practical situations, one would only be interested in sets with a representation complexity lower than five, so Algorithm~\ref{alg:dpmincovset} in general computes the exact solution to (\ref{eq:mincovset}) in an efficient manner.    
  
The pseudocode for calculating the threshold in (\ref{eq:nested:tau}) and set-valued predictions in (\ref{eq:crsvpr}) is presented in Algorithm~\ref{alg:crsvpr:cal} and \ref{alg:crsvpr}. The computational complexity during test time is dominated by the computation of the set of lowest common ancestors in (\ref{eq:mincovset}), which needs to be computed for every $k$ until the stopping criterion in Algorithm~\ref{alg:crsvpr} is satisfied. Furthermore, the worst-case time complexity of calculating the set of lowest common ancestors in Algorithm~\ref{alg:dpmincovset} is upper bounded by $\calO(K^{2}r^{d})$, with $d$ the maximum out-degree of the tree $\calT$. This is small for regimes that are practically useful (fixed $r \leq 3$ and moderate 
$d$).

\section{Experimental results}
\label{sec:results}

\begin{table}[t]
    \caption{Summary statistics and top-1 performance for all datasets. Notation: $K$ -- number of classes, $N$ -- number of samples, Top-1 acc. -- top-1 accuracy of classifier.}
      \label{tab:results:emp}
        \centering
        {\scriptsize
        \begin{tabular}{lcccc|c}
        \toprule
        \textsc{Dataset} & $K$ & $N_{train}$ & $N_{cal}$ & $N_{test}$ & \textsc{Top-1 acc.} \\
        \midrule
        \textsc{CIFAR-10} & 10 & 50000 & 5000 & 5000 & 0.8817 \\
        \textsc{AMB} & 93 & 12781 & 1918 & 1917 & 0.8503 \\
        \textsc{Caltech-101} & 97 & 4338 & 2169 & 2169 & 0.9039 \\
        \textsc{DBpedia} & 219 & 276945 & 30397 & 30397 & 0.9388 \\
        \textsc{Caltech-256} & 256 & 14890 & 7445 & 7445 & 0.7578 \\
        \textsc{PlantCLEF 2015} & 1000 & 91758 & 10723 & 10723 & 0.4156 \\
        \bottomrule
        \end{tabular}
        }
\end{table}

In this section, different set-valued predictors are compared for solving problem (\ref{eq:problem}) across six benchmark datasets, focusing on coverage, efficiency, and representation complexity. Summary statistics for these datasets are presented in Table~\ref{tab:results:emp}. For all datasets, a predefined hierarchy $\calT$ is extracted from the provided hierarchical labels. For detailed information about the experimental setup, we refer to Section~\ref{sec:supp:expsetup} in the technical appendix.

\begin{table*}[t]
  \caption{Results for CIFAR-10, AMB, Caltech-101, DBpedia, Caltech-256 and PlantCLEF 2015. Coverage, efficiency, and representation complexity for the following unrestricted set-valued predictors: LAC, NPS, APS, and restricted set-valued predictors: NCRSVP, CRSVP, NCRSVP-$3$, and CRSVP-$3$. The confidence level is set to 90\%, and calibration and test sets are resampled 10 times.}
  \label{tab:results}
  \centering
  {\scriptsize
  \begin{tabular}{lc|ccccccccccccccc}
      \toprule
      & \textsc{Alg.} & \textsc{LAC} & \textsc{NPS} & \textsc{APS} & \textsc{NCRSVP} & \textsc{CRSVP} & \textsc{NCRSVP-3} & \textsc{CRSVP-3} \\
      \textsc{Dataset} & & & & & & & & & \\
      \midrule
      \multirow{3}{*}{\textsc{CIFAR-10}} & \textsc{Cov.} & $0.899 \pm 0.005$ & $0.997 \pm 0.001$ & $0.899 \pm 0.003$ & $1.000 \pm 0.000$ & $0.899 \pm 0.005$ & $0.997 \pm 0.001$ & $0.899 \pm 0.003$ \\
      & \textsc{Size} & $1.473 \pm 0.023$ & $5.125 \pm 0.058$ & $1.849 \pm 0.019$ & $10.00 \pm 0.000$  & $3.899 \pm 0.049$ & $5.861 \pm 0.064$ & $1.946 \pm 0.025$ \\
      & \textsc{R.C.} & $1.451 \pm 0.021$ & $3.552 \pm 0.017$ & $1.824 \pm 0.015$ & $1.000 \pm 0.000$ & $1.000 \pm 0.000$ & $2.368 \pm 0.008$ & $1.691 \pm 0.009$ \\
      \midrule
      \multirow{3}{*}{\textsc{AMB}} & \textsc{Cov.} 
      & $0.899 \pm 0.010$ & $1.000 \pm 0.000$ & $0.900 \pm 0.009$ & $1.000 \pm 0.000$ & $0.900 \pm 0.009$ & $1.000 \pm 0.000$ & $0.899 \pm 0.009$ \\
      & \textsc{Size} 
      & $1.128 \pm 0.020$ & $23.98 \pm 1.195$ & $1.685 \pm 0.055$ & $50.85 \pm 1.304$ & $4.856 \pm 0.261$ & $39.26 \pm 1.574$ & $2.184 \pm 0.063$ \\
      & \textsc{R.C.} 
      & $1.132 \pm 0.019$ & $17.38 \pm 0.702$ & $1.776 \pm 0.051$ & $1.000 \pm 0.000$ & $1.000 \pm 0.000$ & $1.829 \pm 0.009$ & $1.394 \pm 0.018$ \\
      \midrule
      \multirow{3}{*}{\textsc{Caltech-101}} & \textsc{Cov.} & $0.900 \pm 0.007$ & $1.000 \pm 0.000$ & $0.900 \pm 0.006$ & $1.000 \pm 0.000$ & $0.901 \pm 0.005$ & $1.000 \pm 0.000$ & $0.901 \pm 0.006$ \\
      & \textsc{Size} & $0.920 \pm 0.008$ & $96.54 \pm 0.065$ & $1.165 \pm 0.015$ & $96.25 \pm 0.049$ & $4.400 \pm 0.254$ & $63.61 \pm 0.517$ & $1.784 \pm 0.086$ \\
      & \textsc{R.C.} & $1.000 \pm 0.000$ & $1.454 \pm 0.072$ & $1.251 \pm 0.015$ & $1.000 \pm 0.000$  & $1.000 \pm 0.000$ & $1.306 \pm 0.008$ & $1.133 \pm 0.007$ \\
      \midrule
      \multirow{3}{*}{\textsc{DBpedia}} & \textsc{Cov.} 
      & $0.899 \pm 0.001$ & $1.000 \pm 0.000$ & $0.901 \pm 0.002$ & $0.999 \pm 0.000$ & $0.901 \pm 0.002$ & $0.999 \pm 0.000$ & $0.901 \pm 0.001$ \\
      & \textsc{Size} 
      & $0.931 \pm 0.002$ & $59.17 \pm 0.272$ & $11.33 \pm 0.287$ & $162.9 \pm 0.492$ & $26.62 \pm 0.417$ & $127.9 \pm 0.501$ & $17.71 \pm 0.423$ \\
      & \textsc{R.C.} 
      & $1.000 \pm 0.000$ & $53.18 \pm 0.257$ & $11.90 \pm 0.277$ & $1.000 \pm 0.000$ & $1.000 \pm 0.000$ & $1.985 \pm 0.004$ & $1.380 \pm 0.005$ \\
      \midrule
      \multirow{3}{*}{\textsc{Caltech-256}} & \textsc{Cov.} & $0.900 \pm 0.004$ & $0.999 \pm 0.000$  & $0.901 \pm 0.004$ & $0.999 \pm 0.000$ & $0.900 \pm 0.003$ & $1.000 \pm 0.000$ & $0.901 \pm 0.003$ \\
      & \textsc{Size} & $1.931 \pm 0.040$ & $62.25 \pm 1.330$ & $3.640 \pm 0.099$ & $208.4 \pm 1.780$ & $44.69 \pm 1.252$ & $163.6 \pm 2.070$ & $20.30 \pm 0.830$ \\
      & \textsc{R.C.} & $1.926 \pm 0.040$ & $50.12 \pm 0.996$ &  $3.680 \pm 0.098$ & $1.000 \pm 0.000$ & $1.000 \pm 0.000$ & $1.632 \pm 0.005$ & $1.498 \pm 0.009$ \\
      \midrule
      \multirow{3}{*}{\textsc{PlantCLEF 2015}} & \textsc{Cov.} & $0.899 \pm 0.003$ & $0.970 \pm 0.001$ & $0.899 \pm 0.003$ & $1.000 \pm 0.000$ & $0.900 \pm 0.002$ & $1.000 \pm 0.000$ & $0.899 \pm 0.004$ \\
      & \textsc{Size} & $25.50 \pm 0.450$ & $123.6 \pm 2.026$ & $44.12 \pm 0.994$ & $998.9 \pm 0.224$ & $520.9 \pm 4.745$ & $995.3 \pm 0.492$ & $389.7 \pm 5.898$ \\
      & \textsc{R.C.} & $24.33 \pm 0.419$  & $104.4 \pm 1.560$ & $40.19 \pm 0.836$ & $1.000 \pm 0.000$ & $1.000 \pm 0.000$ & $1.006 \pm 0.001$ & $1.632 \pm 0.010$ \\
      \bottomrule
  \end{tabular}
  }
\end{table*}

In particular, we will compare various set-valued predictors that solve problem (\ref{eq:problem}) across six datasets listed in Table~\ref{tab:results:emp}: CIFAR-10~\citep{krizhevsky10cifar10}, Caltech-101~\citep{li03caltech101}, Caltech-256~\citep{griffin07caltech256}, PlantCLEF 2015~\citep{goeau15lifeclef}, the Allen Mouse Brain (AMB) single-cell transcriptomics dataset~\citep{tasic2018shared} and DBPedia~\citep{ofer19dbpedia}. The first four datasets are image datasets from the computer vision domain. The AMB dataset contains gene expression profiles from the mouse neocortex along with corresponding cell types and a hierarchical structure. The DBpedia text dataset contains extracted structured textual content from Wikipedia. Specifically, we evaluate the following set-valued predictors: CRSVP and CRSVP-$3$. Additionally, we demonstrate the usefulness of randomized prediction sets by considering the following na\"ive (N) set-valued predictors: NCRSVP and NCRSVP-$3$. These correspond to setting $u$ to zero when calculating the threshold in (\ref{eq:nested:tau}) and constructing the prediction sets in (\ref{eq:crsvp}) and (\ref{eq:crsvpr}), respectively. Finally, we include results for three baseline methods that produce valid set-valued predictions in flat classification (i.e.\ ignoring the hierarchy): the least ambiguous classifier (LAC), as proposed by \citet{sadinle19least}; adaptive prediction sets (APS), as defined in (\ref{eq:aps}); and na\"ive prediction sets (NPS), which correspond to APS without randomization (i.e.\ setting $u$ to zero). The results presented below are obtained using a flat classifier. We do not use hierarchical classifiers, as our dynamic programming approach relies on a bottom-up traversal of the tree, and incorporating hierarchical classifiers would increase computational complexity. Moreover, \citet{mortier22svprc} report no significant improvement in predictive performance when using hierarchical classifiers compared to flat ones. Nevertheless, hierarchical classifiers can still be applied within our framework, as discussed in Section~\ref{sec:hierclass}. During inference, we resample the calibration and test sets 10 times and use a confidence level of 90\%. 

The most important results are summarized in Table~\ref{tab:results}. For each experiment, we report (marginal) coverage, efficiency, and representation complexity. Coverage is defined as the proportion of samples for which the true class is included in the prediction set. Efficiency is defined as the average size of the set-valued prediction. The results clearly indicate that na\"ive set-valued predictors fail to deliver prediction sets with exact coverage, thereby highlighting the importance of randomized prediction sets. Moreover, increasing the representation complexity generally improves efficiency, demonstrating its practical value. In extreme cases, when representation complexity is unrestricted, such as with LAC, NPS, and APS, optimal performance in terms of efficiency is observed. However, this comes at the cost of significantly increased representation complexity in the prediction sets, in particular for large $K$, which may not be practical when predictions need to adhere to a predefined hierarchy. 

Note that imprecise predictions for low representation complexity are typically due to uncertainty in the predicted probabilities and/or characteristics of the hierarchical structure, such as inconsistency and depth. Specifically, when uncertainty spans distinct branches of the hierarchy, predictions tend to correspond to higher-level nodes for lower representation complexity. This suggests that the hierarchy may not be well-suited for certain samples. Additionally, the depth of the hierarchical structure also plays a critical role in hierarchical classification tasks with a large number of classes. In such cases, a shallow hierarchy can lead to imprecise predictions for low representation complexity, as internal nodes in shallow trees often have many children. For example, the PlantCLEF 2015 dataset, characterized by 1,000 classes and a shallow hierarchy (i.e.\ taxonomy consisting of a family, genus and species level), requires a higher representation complexity in order to improve efficiency. 
  
Finally, higher representation complexity becomes valuable for datasets with many classes and high uncertainty. In this case, traditional hierarchical classification tends to return imprecise predictions, i.e.\ nodes higher up in the hierarchy. By restricting representation complexity, predictions can span multiple nodes lower in the hierarchy, thereby improving efficiency. This is demonstrated by comparing CRSVP and CRSVP-3 on the PlantCLEF 2015 dataset in Table~\ref{tab:results:emp} and Figure~\ref{fig:repvssize}. Furthermore, Figure~\ref{fig:repvssize} reveals a clear trade-off between representation complexity and efficiency for the PlantCLEF 2015 dataset. In Section~\ref{sec:supp:results} of the technical appendix, we present additional results obtained for (N)CRSVP-$1$ and (N)CRSVP-$2$, and provide additional insights regarding conditional coverage and the relationship between representation complexity and efficiency. 

\begin{figure}[t]
  \centering
\includegraphics[width=0.8\columnwidth]{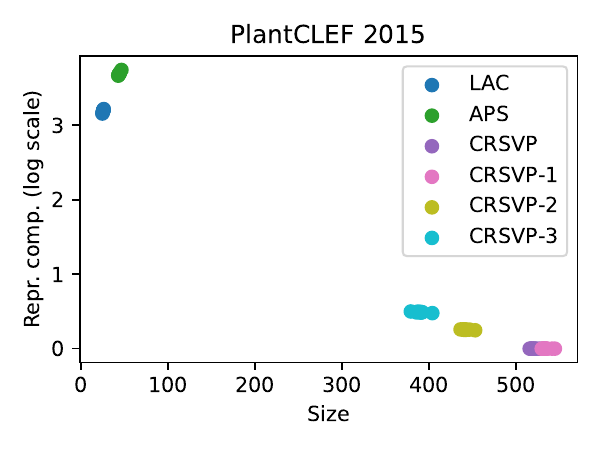}
\caption{Trade-off between representation complexity (log scale) and efficiency for PlantCLEF 2015. The confidence level is set to 90\%, and calibration and test sets are resampled 10 times.}
\label{fig:repvssize}
\end{figure}

\section{Conclusion}
\label{sec:conclusion}

In this work, we extended the split conformal prediction framework to hierarchical classification by introducing two novel set-valued prediction algorithms. The first algorithm generates valid set-valued predictions restricted to single nodes within a predefined hierarchy. We argued that this restriction can be limiting in certain applications. To address this limitation, we introduced a second algorithm that relaxes this constraint by incorporating the notion of representation complexity. Empirical evaluations on multiple benchmark datasets demonstrate the effectiveness of the proposed algorithms in achieving exact nominal coverage. 

The main purpose of bounding representation complexity is to ensure the interpretability and semantic meaningfulness of predictions. However, it may also have other benefits. In particular, we conjecture that it may serve the purpose of regularization and hence increase accuracy in cases where class probabilities are poorly estimated. In such cases, flat prediction sets tend to be scattered across the entire hierarchy, and this scattering is suppressed by bounding representation complexity. We plan to elaborate on this aspect more closely in future work. 
Another interesting direction for future research is to generalize our methods to more complex structures, such as directed acyclic graphs.

\subsubsection*{Acknowledgements}
Yusuf Sale is supported by the DAAD program Konrad Zuse Schools of Excellence in Artificial Intelligence, sponsored by the Federal Ministry of Education and Research. Alireza Javanmardi gratefully acknowledges funding by the Klaus Tschira Stiftung (project 00.019.2024). Willem Waegeman received funding from the Flemish government under the Flanders AI research (FlAIR) program. 

\bibliographystyle{apalike}
\bibliography{main}

\begin{thebibliography}{}

\bibitem[Aho et~al., 1973]{aho73finding}
Aho, A.~V., Hopcroft, J.~E., and Ullman, J.~D. (1973).
\newblock On finding lowest common ancestors in trees.
\newblock In {\em Proceedings of the fifth annual ACM symposium on Theory of
  computing}, pages 253--265.

\bibitem[Alex~Freitas, 2007]{freitas07hc}
Alex~Freitas, A. d.~C. (2007).
\newblock A tutorial on hierarchical classification with applications in
  bioinformatics.
\newblock In {\em Research and Trends in Data Mining Technologies and
  Applications,}, pages 175--208.

\bibitem[Angelopoulos et~al., 2023]{angelopoulos23conformalriskcontrol}
Angelopoulos, A., Bates, S., Fisch, A., Lei, L., and Schuster, T. (2023).
\newblock Conformal risk control.

\bibitem[Angelopoulos et~al., 2020]{angelopoulos20raps}
Angelopoulos, A.~N., Bates, S., Malik, J., and Jordan, M.~I. (2020).
\newblock Uncertainty sets for image classifiers using conformal prediction.
\newblock {\em ArXiv}, abs/2009.14193.

\bibitem[Bender and Farach-Colton, 2000]{bender00lca}
Bender, M.~A. and Farach-Colton, M. (2000).
\newblock The lca problem revisited.
\newblock In {\em LATIN 2000: Theoretical Informatics: 4th Latin American
  Symposium, Punta del Este, Uruguay, April 10-14, 2000 Proceedings 4}, pages
  88--94. Springer.

\bibitem[Beygelzimer et~al., 2009]{beygelzimer09cpte}
Beygelzimer, A., Langford, J., Lifshits, Y., Sorkin, G., and Strehl, A. (2009).
\newblock Conditional probability tree estimation analysis and algorithms.
\newblock In {\em Proceedings of the Twenty-Fifth Conference on Uncertainty in
  Artificial Intelligence}, UAI '09, pages 51--58.

\bibitem[Bi and Kwok, 2015]{bi15bohmlc}
Bi, W. and Kwok, J. (2015).
\newblock Bayes-optimal hierarchical multilabel classification.
\newblock {\em IEEE Transactions on Knowledge and Data Engineering}, 27.

\bibitem[Cauchois et~al., 2021]{cauchois21knowing}
Cauchois, M., Gupta, S., and Duchi, J.~C. (2021).
\newblock Knowing what you know: valid and validated confidence sets in
  multiclass and multilabel prediction.
\newblock {\em Journal of machine learning research}, 22(81):1--42.

\bibitem[Dash et~al., 2013]{dash13scalable}
Dash, S.~K., Scholz, S.-B., Herhut, S., and Christianson, B. (2013).
\newblock A scalable approach to computing representative lowest common
  ancestor in directed acyclic graphs.
\newblock {\em Theoretical Computer Science}, 513:25--37.

\bibitem[Dembczy{\'n}ski et~al., 2016]{dembczynski16consistencyop}
Dembczy{\'n}ski, K., Kot\l{}owski, W., Waegeman, W., Busa-Fekete, R., and
  H{\"u}llermeier, E. (2016).
\newblock Consistency of probabilistic classifier trees.
\newblock In {\em ECML/PKDD}.

\bibitem[Dembczy{\'n}ski et~al., 2012]{dembczynski12chaining}
Dembczy{\'n}ski, K., Waegeman, W., and H{\"u}llermeier, E. (2012).
\newblock An analysis of chaining in multi-label classification.
\newblock In {\em ECAI 2012}, pages 294--299. IOS Press.

\bibitem[Feldman et~al., 2021]{feldman21improving}
Feldman, S., Bates, S., and Romano, Y. (2021).
\newblock Improving conditional coverage via orthogonal quantile regression.
\newblock {\em Advances in neural information processing systems},
  34:2060--2071.

\bibitem[Fox, 1997]{fox97reg}
Fox, J. (1997).
\newblock {\em Applied regression analysis, linear models, and related
  methods}.
\newblock Sage.

\bibitem[Foygel~Barber et~al., 2021]{foygel21limits}
Foygel~Barber, R., Candes, E.~J., Ramdas, A., and Tibshirani, R.~J. (2021).
\newblock The limits of distribution-free conditional predictive inference.
\newblock {\em Information and Inference: A Journal of the IMA},
  10(2):455--482.

\bibitem[Frank and Kramer, 2004]{frank04nested}
Frank, E. and Kramer, S. (2004).
\newblock Ensembles of nested dichotomies for multi-class problems.
\newblock In {\em Proceedings of the Twenty-first International Conference on
  Machine Learning}, ICML '04.

\bibitem[Gao et~al., 2025]{gao2025volume}
Gao, C., Shan, L., Srinivas, V., and Vijayaraghavan, A. (2025).
\newblock Volume optimality in conformal prediction with structured prediction
  sets.
\newblock In {\em Forty-second International Conference on Machine Learning}.

\bibitem[G{\"o}eau et~al., 2015]{goeau15lifeclef}
G{\"o}eau, H., Joly, A., and Pierre, B. (2015).
\newblock Lifeclef plant identification task 2015.
\newblock {\em CLEF Working Notes}, 2015.

\bibitem[Goren et~al., 2024]{goren24hierarchical}
Goren, S., Galil, I., and El-Yaniv, R. (2024).
\newblock Hierarchical selective classification.
\newblock In {\em The Thirty-eighth Annual Conference on Neural Information
  Processing Systems}.

\bibitem[Griffin et~al., 2007]{griffin07caltech256}
Griffin, G., Holub, A., and Perona, P. (2007).
\newblock Caltech-256 object category dataset.
\newblock Technical Report 7694, California Institute of Technology.

\bibitem[Gupta et~al., 2022]{gupta22nested}
Gupta, C., Kuchibhotla, A.~K., and Ramdas, A. (2022).
\newblock Nested conformal prediction and quantile out-of-bag ensemble methods.
\newblock {\em Pattern Recognition}, 127:108496.

\bibitem[Harel and Tarjan, 1984]{harel84fast}
Harel, D. and Tarjan, R.~E. (1984).
\newblock Fast algorithms for finding nearest common ancestors.
\newblock {\em siam Journal on Computing}, 13(2):338--355.

\bibitem[Krizhevsky et~al., 2010]{krizhevsky10cifar10}
Krizhevsky, A., Nair, V., and Hinton, G.~E. (2010).
\newblock Cifar-10 (canadian institute for advanced research).
\newblock Technical report, Canadian Institute for Advanced Research.

\bibitem[Krizhevsky et~al., 2017]{krizhevsky17imagenet}
Krizhevsky, A., Sutskever, I., and Hinton, G.~E. (2017).
\newblock Imagenet classification with deep convolutional neural networks.
\newblock {\em Communications of the ACM}, 60(6):84--90.

\bibitem[Kumar et~al., 2013]{kumar13beam}
Kumar, A., Vembu, S., Menon, A.~K., and Elkan, C. (2013).
\newblock Beam search algorithms for multilabel learning.
\newblock {\em Machine learning}, 92:65--89.

\bibitem[Lei et~al., 2018]{lei18distribution}
Lei, J., G'Sell, M., Rinaldo, A., Tibshirani, R.~J., and Wasserman, L. (2018).
\newblock Distribution-free predictive inference for regression.
\newblock {\em Journal of the American Statistical Association},
  113(523):1094--1111.

\bibitem[Li et~al., 2003]{li03caltech101}
Li, F.-F., Andreetto, M., and Ranzato, M. (2003).
\newblock Caltech-101 image data\-set.
\newblock Technical report, California Institute of Technology.

\bibitem[Melnikov and H{\"{u}}llermeier, 2018]{melnikov18nested}
Melnikov, V. and H{\"{u}}llermeier, E. (2018).
\newblock On the effectiveness of heuristics for learning nested dichotomies:
  an empirical analysis.
\newblock {\em Machine Learning}, 107(8--10):1537--1560.

\bibitem[Mena~Waldo et~al., 2015]{mena15astar}
Mena~Waldo, D., Monta{\~n}{\'e}s~Roces, E., Quevedo~P{\'e}rez, J.~R.,
  Coz~Velasco, J. J.~d., et~al. (2015).
\newblock Using a* for inference in probabilistic classifier chains.
\newblock In {\em Proceedings of the Twenty-Fourth International Joint
  Conference on Artificial Intelligence (IJCAI 2015)}. Association for the
  Advancement of Artificial Intelligence.

\bibitem[Morin and Bengio, 2005]{morin05hiers}
Morin, F. and Bengio, Y. (2005).
\newblock Hierarchical probabilistic neural network language model.
\newblock In {\em Proceedings of the Tenth International Workshop on Artificial
  Intelligence and Statistics}, pages 246--252. Society for Artificial
  Intelligence and Statistics.

\bibitem[Mortier et~al., 2022]{mortier22svprc}
Mortier, T., H{\"u}llermeier, E., Dembczy{\'n}ski, K., and Waegeman, W. (2022).
\newblock Set-valued prediction in hierarchical classification with constrained
  representation complexity.
\newblock In {\em Uncertainty in Artificial Intelligence}, pages 1392--1401.
  PMLR.

\bibitem[Mortier et~al., 2021]{mortier21efficientsvp}
Mortier, T., Wydmuch, M., Dembczy\'nski, K., and Waegeman, W. (2021).
\newblock Efficient set-valued prediction in multi-class classification.
\newblock {\em KDD}, 35:1435--1469.

\bibitem[Ofer, 2019]{ofer19dbpedia}
Ofer, D. (2019).
\newblock Dbpedia classes.

\bibitem[Oh, 2017]{oh17topkhc}
Oh, S. (2017).
\newblock Top-k hierarchical classification.
\newblock In {\em {AAAI}}, pages 2450--2456. {AAAI} Press.

\bibitem[Papadopoulos et~al., 2002]{papadopoulos02inductive}
Papadopoulos, H., Proedrou, K., Vovk, V., and Gammerman, A. (2002).
\newblock Inductive confidence machines for regression.
\newblock In {\em Machine learning: ECML 2002: 13th European conference on
  machine learning Helsinki, Finland, August 19--23, 2002 proceedings 13},
  pages 345--356. Springer.

\bibitem[Paszke et~al., 2019]{paszke17automatic}
Paszke, A., Gross, S., Massa, F., Lerer, A., Bradbury, J., Chanan, G., Killeen,
  T., Lin, Z., Gimelshein, N., Antiga, L., Desmaison, A., Kopf, A., Yang, E.,
  DeVito, Z., Raison, M., Tejani, A., Chilamkurthy, S., Steiner, B., Fang, L.,
  Bai, J., and Chintala, S. (2019).
\newblock Pytorch: An imperative style, high-performance deep learning library.
\newblock In {\em Advances in Neural Information Processing Systems 32}, pages
  8026--8037.

\bibitem[Rangwala and Naik, 2017]{rangwala17lhc}
Rangwala, H. and Naik, A. (2017).
\newblock Large scale hierarchical classification: foundations, algorithms and
  applications.
\newblock In {\em The European Conference on ML and Principles and Practice of
  Knowledge Discovery in Databases}.

\bibitem[Romano et~al., 2019]{romano19conformalized}
Romano, Y., Patterson, E., and Candes, E. (2019).
\newblock Conformalized quantile regression.
\newblock {\em Advances in neural information processing systems}, 32.

\bibitem[Romano et~al., 2020]{romano20classification}
Romano, Y., Sesia, M., and Candes, E. (2020).
\newblock Classification with valid and adaptive coverage.
\newblock {\em Advances in Neural Information Processing Systems},
  33:3581--3591.

\bibitem[Rossellini et~al., 2024]{rossellini24integrating}
Rossellini, R., Barber, R.~F., and Willett, R. (2024).
\newblock Integrating uncertainty awareness into conformalized quantile
  regression.
\newblock In {\em International Conference on Artificial Intelligence and
  Statistics}, pages 1540--1548. PMLR.

\bibitem[Sadinle et~al., 2019]{sadinle19least}
Sadinle, M., Lei, J., and Wasserman, L. (2019).
\newblock Least ambiguous set-valued classifiers with bounded error levels.
\newblock {\em Journal of the American Statistical Association},
  114(525):223--234.

\bibitem[Sandler et~al., 2018]{sandler18mobilenetv2}
Sandler, M., Howard, A., Zhu, M., Zhmoginov, A., and Chen, L.-C. (2018).
\newblock Mobilenetv2: Inverted residuals and linear bottlenecks.
\newblock In {\em Proceedings of the IEEE conference on computer vision and
  pattern recognition}, pages 4510--4520.

\bibitem[Tasic et~al., 2018]{tasic2018shared}
Tasic, B., Yao, Z., Graybuck, L.~T., Smith, K.~A., Nguyen, T.~N., Bertagnolli,
  D., Goldy, J., Garren, E., Economo, M.~N., Viswanathan, S., et~al. (2018).
\newblock Shared and distinct transcriptomic cell types across neocortical
  areas.
\newblock {\em Nature}, 563(7729):72--78.

\bibitem[Theunissen et~al., 2024]{theunissen2024uncertainty}
Theunissen, L., Mortier, T., Saeys, Y., and Waegeman, W. (2024).
\newblock Uncertainty-aware single-cell annotation with a hierarchical reject
  option.
\newblock {\em Bioinformatics}, 40(3):btae128.

\bibitem[Valmadre, 2022]{valmadre22hierarchical}
Valmadre, J. (2022).
\newblock Hierarchical classification at multiple operating points.
\newblock {\em Advances in Neural Information Processing Systems},
  35:18034--18045.

\bibitem[Vovk et~al., 2005]{vovk05algorithmic}
Vovk, V., Gammerman, A., and Shafer, G. (2005).
\newblock {\em Algorithmic learning in a random world}, volume~29.
\newblock Springer.

\bibitem[Wang et~al., 2021]{wang21hierarchical}
Wang, Y., Wang, Z., Hu, Q., Zhou, Y., and Su, H. (2021).
\newblock Hierarchical semantic risk minimization for large-scale
  classification.
\newblock {\em IEEE Transactions on Cybernetics}, 52(9):9546--9558.

\bibitem[World Health~Organization, 1978]{who78icd9}
World Health~Organization, e.~a. (1978).
\newblock {\em International classification of diseases:[9th] ninth revision,
  basic tabulation list with alphabetic index}.
\newblock World Health Organization.

\bibitem[Yang et~al., 2017]{yang17cautioushmc}
Yang, G., Destercke, S., and Masson, M.-H. (2017).
\newblock Cautious classification with nested dichotomies and imprecise
  probabilities.
\newblock {\em Soft Computing}, 21:7447--7462.

\end{thebibliography}

\clearpage
\appendix
\thispagestyle{empty}


\onecolumn
\aistatstitle{Conformal Prediction in Hierarchical Classification with Constrained Representation Complexity: \\
Supplementary Materials}


\section{Experimental setup}
\label{sec:supp:expsetup}

We use a MobileNetV2 convolutional neural network \citep{sandler18mobilenetv2} pretrained on ImageNet~\citep{krizhevsky17imagenet}, in order to obtain hidden representations for all image datasets. The cross-entropy loss is minimized using the Adam optimizer, with a learning rate of $1 \times 10^{-5}$ and momentum set to 0.99. We set the number of epochs to 2 and 20, for the Caltech and other datasets, respectively. We train all models end-to-end on a GPU, by using the PyTorch library~\citep{paszke17automatic} and infrastructure with the following specifications:
\begin{itemize} 
\item CPU: Intel i7-6800K 3.4 GHz (3.8 GHz Turbo Boost) 
\item GPU: NVIDIA GTX 1080 Ti 11GB 
\item RAM: 64GB DDR4-2666 
\end{itemize}
For the AMB and DBpedia dataset, we use the same model and training procedure as described in~\citet{theunissen2024uncertainty} and ~\citet{mortier21efficientsvp}, respectively. 
 
\section{Additional results}
\label{sec:supp:results}

In this section, we include additional results to those presented in Table~\ref{tab:results}, namely, results for (N)CRSVP-$1$ and (N)CRSVP-$2$. In addition, we give insights into potential conditional coverage violations by examining the Pearson correlation between coverage and prediction set size. A strong correlation may indicate a potential violation of conditional coverage~\citep{feldman21improving}. While this metric serves as a useful proxy for evaluating conditional coverage, it is not definitive, as a correlation of zero does not necessarily guarantee conditional coverage, as shown in~\citet{rossellini24integrating}. The results are summarized in Table~\ref{tab:supp:extresults1}, \ref{tab:supp:extresults2} and~\ref{tab:supp:extresults3}. 

First, note that CRSVP is different from CRSVP-$1$. The restriction is the same, however, the set of solutions for both algorithms is different: For CRSVP, starting from the mode of the distribution, the solution space is given by the mode and its ancestors. For CRSVP-$1$, the solution space is given by line 7 and 6 in Algorithm~\ref{alg:crsvp} and ~\ref{alg:crsvpr:cal}, respectively, namely, the most common ancestor (\ref{eq:mincovset}) for the top-1, top-2, top-3, etc., classes. In most cases, we observe low correlation for (N)CRSVP and (N)CRSVP-$r$, reflecting strong performance regarding conditional coverage. 

Finally, in Figure~\ref{fig:supp:extresults}, we visually explore the relationship between representation complexity and set size for the unrestricted set-valued predictors (LAC, APS) and restricted set-valued predictors (CRSVP, CRSVP-$r$). As is clear from both Table~\ref{tab:supp:extresults1},~\ref{tab:supp:extresults2},~\ref{tab:supp:extresults3} and Figure~\ref{fig:supp:extresults}, increasing the representation complexity results in better efficiency. 

\begin{table*}[tb]
  \renewcommand{\thetable}{C\arabic{table}}
  \caption{Additional results for CIFAR-10 and AMB. Coverage, efficiency, representation complexity and Pearson correlation values for the following unrestricted set-valued predictors: LAC, NPS, APS, and restricted set-valued predictors: NCRSVP, CRSVP, NCRSVP-$r$, and CRSVP-$r$. The confidence level is set to 90\%, and calibration and test sets are resampled 10 times.}
  \label{tab:supp:extresults1}
  \centering
  {\scriptsize
  \begin{tabular}{c|cccc|cccc}
      \toprule
      \textsc{Dataset} & \multicolumn{4}{c|}{\textsc{CIFAR-10}} & \multicolumn{4}{c}{\textsc{AMB}} \\
      \textsc{Alg.} & \textsc{Cov.} & \textsc{Size} & \textsc{Repr. comp.} & \textsc{Corr.} & 
      \textsc{Cov.} & \textsc{Size} & \textsc{Repr. comp.} & \textsc{Corr.} \\
      \midrule
      \textsc{LAC} & 
      $0.899 \pm 0.005$ & $1.473 \pm 0.023$ & $1.451 \pm 0.021$ & $-0.07 \pm 0.011$ & 
      $0.899 \pm 0.010$ & $1.128 \pm 0.020$ & $1.132 \pm 0.019$ & $0.060 \pm 0.025$ \\
      \textsc{NPS} &
      $0.997 \pm 0.001$ & $5.125 \pm 0.058$ & $3.552 \pm 0.017$ & $0.016 \pm 0.010$ & 
      $1.000 \pm 0.000$ & $23.98 \pm 1.195$ & $17.38 \pm 0.702$ & - \\
      \textsc{APS} &
      $0.899 \pm 0.003$ & $1.849 \pm 0.019$ & $1.824 \pm 0.015$ & $0.111 \pm 0.014$ & 
      $0.900 \pm 0.009$ & $1.685 \pm 0.055$ & $1.776 \pm 0.051$ & $0.177 \pm 0.015$ \\
      \midrule
      \textsc{NCRSVP} &
      $1.000 \pm 0.000$ & $10.00 \pm 0.000$ & $1.000 \pm 0.000$ & - & 
      $1.000 \pm 0.000$ & $50.85 \pm 1.304$ & $1.000 \pm 0.000$ & - \\
      \textsc{CRSVP} &
      $0.899 \pm 0.005$ & $3.899 \pm 0.049$ & $1.000 \pm 0.000$ & $0.134 \pm 0.009$ & 
      $0.900 \pm 0.009$ & $4.856 \pm 0.261$ & $1.000 \pm 0.000$ & $0.137 \pm 0.008$ \\
      \midrule
      \textsc{NCRSVP-1} &
       $1.000 \pm 0.000$ & $10.00 \pm 0.000$ & $1.000 \pm 0.000$ & - &
       $1.000 \pm 0.000$ & $51.13 \pm 1.338$ & $1.000 \pm 0.000$ & - \\
      \textsc{CRSVP-1} &
       $0.898 \pm 0.005$ & $3.580 \pm 0.050$ & $1.000 \pm 0.000$ & $0.207 \pm 0.005$ &
       $0.900 \pm 0.008$ & $4.890 \pm 0.266$ & $1.000 \pm 0.000$ & $0.148 \pm 0.007$ \\
      \midrule
      \textsc{NCRSVP-2} &
      $0.998 \pm 0.001$ & $6.775 \pm 0.047$ & $1.616 \pm 0.004$ & $0.028 \pm 0.009$ &
       $1.000 \pm 0.000$ & $43.35 \pm 1.578$ & $1.511 \pm 0.011$ & - \\
      \textsc{CRSVP-2} &
       $0.899 \pm 0.005$ & $2.279 \pm 0.036$ & $1.457 \pm 0.005$ & $0.105 \pm 0.012$ 
      & $0.900 \pm 0.009$ & $2.666 \pm 0.117$ & $1.248 \pm 0.011$ & $0.124 \pm 0.007$ \\
      \midrule
      \textsc{NCRSVP-3} &
      $0.997 \pm 0.001$ & $5.861 \pm 0.064$ & $2.368 \pm 0.008$ & $0.024 \pm 0.009$ & 
      $1.000 \pm 0.000$ & $39.26 \pm 1.574$ & $1.829 \pm 0.009$ & - \\
      \textsc{CRSVP-3} &
      $0.899 \pm 0.003$ & $1.946 \pm 0.025$ & $1.691 \pm 0.009$ & $0.103 \pm 0.014$ & 
      $0.899 \pm 0.009$ & $2.184 \pm 0.063$ & $1.394 \pm 0.018$ & $0.114 \pm 0.009$ \\
      \bottomrule
  \end{tabular}
  }
\end{table*}

\begin{table*}[tb]
  \renewcommand{\thetable}{C\arabic{table}}
  \caption{Additional results for Caltech-101 and DBpedia. Coverage, efficiency, representation complexity and Pearson correlation values for the following unrestricted set-valued predictors: LAC, NPS, APS, and restricted set-valued predictors: NCRSVP, CRSVP, NCRSVP-$r$, and CRSVP-$r$. The confidence level is set to 90\%, and calibration and test sets are resampled 10 times.}
  \label{tab:supp:extresults2}
  \centering
  {\scriptsize
  \begin{tabular}{c|cccc|cccc}
      \toprule
      \textsc{Dataset} & \multicolumn{4}{c|}{\textsc{Caltech-101}} & \multicolumn{4}{c}{\textsc{DBpedia}} \\
      \textsc{Alg.} & \textsc{Cov.} & \textsc{Size} & \textsc{Repr. comp.} & \textsc{Corr.} & 
      \textsc{Cov.} & \textsc{Size} & \textsc{Repr. comp.} & \textsc{Corr.} \\
      \midrule
      \textsc{LAC} & 
      $0.900 \pm 0.007$ & $0.920 \pm 0.008$ & $1.000 \pm 0.000$ & $0.886 \pm 0.021$ & 
      $0.899 \pm 0.001$ & $0.931 \pm 0.002$ & $1.000 \pm 0.000$ & $0.814 \pm 0.005$ \\
      \textsc{NPS} &
      $1.000 \pm 0.000$ & $96.54 \pm 0.065$ & $1.454 \pm 0.072$ & - & 
      $1.000 \pm 0.000$ & $59.17 \pm 0.272$ & $53.18 \pm 0.257$ & $-0.010 \pm 0.003$ \\
      \textsc{APS} &
      $0.900 \pm 0.006$ & $1.165 \pm 0.015$ & $1.251 \pm 0.015$ & $0.184 \pm 0.013$ & 
      $0.901 \pm 0.002$ & $11.33 \pm 0.287$ & $11.90 \pm 0.277$ & $0.129 \pm 0.001$ \\
      \midrule
      \textsc{NCRSVP} &
      $1.000 \pm 0.000$ & $96.25 \pm 0.049$ & $1.000 \pm 0.000$ & - & 
      $0.999 \pm 0.000$ & $162.9 \pm 0.492$ & $1.000 \pm 0.000$ & $0.019 \pm 0.001$ \\
      \textsc{CRSVP} &
      $0.901 \pm 0.005$ & $4.400 \pm 0.254$ & $1.000 \pm 0.000$ & $0.062 \pm 0.009$ & 
      $0.901 \pm 0.002$ & $26.62 \pm 0.417$ & $1.000 \pm 0.000$ & $0.087 \pm 0.002$ \\
      \midrule
      \textsc{NCRSVP-1} &
       $1.000 \pm 0.000$ & $70.36 \pm 0.550$ & $1.000 \pm 0.000$ & - &
       $1.000 \pm 0.000$ & $183.2 \pm 0.380$ & $1.000 \pm 0.000$ & $0.023 \pm 0.003$ \\
      \textsc{CRSVP-1} &
       $0.900 \pm 0.005$ & $4.173 \pm 0.241$ & $1.000 \pm 0.000$ & $0.079 \pm 0.006$ &
       $0.901 \pm 0.002$ & $25.59 \pm 0.386$ & $1.000 \pm 0.000$ & $0.119 \pm 0.002$ \\
      \midrule
      \textsc{NCRSVP-2} &
      $1.000 \pm 0.000$ & $66.39 \pm 0.528$ & $1.115 \pm 0.004$ & - &
       $0.999 \pm 0.000$ & $150.6 \pm 0.418$ & $1.351 \pm 0.002$ & $0.013 \pm 0.001$ \\
      \textsc{CRSVP-2} &
       $0.900 \pm 0.007$ & $2.381 \pm 0.110$ & $1.085 \pm 0.004$ & $0.054 \pm 0.014$ 
      & $0.901 \pm 0.002$ & $19.34 \pm 0.484$ & $1.191 \pm 0.002$ & $0.107 \pm 0.001$ \\
      \midrule
      \textsc{NCRSVP-3} &
      $1.000 \pm 0.000$ & $63.61 \pm 0.517$ & $1.306 \pm 0.008$ & - & 
      $0.999 \pm 0.000$ & $127.9 \pm 0.501$ & $1.985 \pm 0.004$ & $0.004 \pm 0.001$ \\
      \textsc{CRSVP-3} &
      $0.901 \pm 0.006$ & $1.784 \pm 0.086$ & $1.133 \pm 0.007$ & $0.066 \pm 0.011$ & 
      $0.901 \pm 0.001$ & $17.71 \pm 0.423$ & $1.380 \pm 0.005$ & $0.113 \pm 0.001$ \\
      \bottomrule
  \end{tabular}
  }
\end{table*}

\begin{table*}[t]
  \renewcommand{\thetable}{C\arabic{table}}
  \caption{Additional results for Caltech-256 and PlantCLEF 2015. Coverage, efficiency, representation complexity and Pearson correlation values for the following unrestricted set-valued predictors: LAC, NPS, APS, and restricted set-valued predictors: NCRSVP, CRSVP, NCRSVP-$r$, and CRSVP-$r$. The confidence level is set to 90\%, and calibration and test sets are resampled 10 times.}
  \label{tab:supp:extresults3}
  \centering
  {\scriptsize
  \begin{tabular}{c|cccc|ccccc}
      \toprule
      \textsc{Dataset} & \multicolumn{4}{c|}{\textsc{Caltech-256}} & \multicolumn{4}{c}{\textsc{PlantCLEF 2015}} \\
      \textsc{Alg.} & \textsc{Cov.} & \textsc{Size} & \textsc{Repr. comp.} & \textsc{Corr.} & 
      \textsc{Cov.} & \textsc{Size} & \textsc{Repr. comp.} & \textsc{Corr.} \\
      \midrule
      \textsc{LAC} &
      $0.900 \pm 0.004$ & $1.931 \pm 0.040$ & $1.926 \pm 0.040$ & $-0.29 \pm 0.006$ & 
      $0.899 \pm 0.003$ & $25.50 \pm 0.450$ & $24.33 \pm 0.419$ & $-0.20\pm 0.009$ \\
      \textsc{NPS} &
      $0.999 \pm 0.000$ & $62.25 \pm 1.330$ & $50.12 \pm 0.996$ & $-0.01 \pm 0.009$ & 
      $0.970 \pm 0.001$ & $123.6 \pm 2.026$ & $104.4 \pm 1.560$ & $0.003 \pm 0.006$ \\
      \textsc{APS} &
      $0.901 \pm 0.004$ & $3.640 \pm 0.099$ & $3.680 \pm 0.098$ & $-0.05 \pm 0.008$ & 
      $0.899 \pm 0.003$ & $44.12 \pm 0.994$ & $40.19 \pm 0.836$ & $-0.01 \pm 0.010$ \\
      \midrule
      \textsc{NCRSVP} &
      $0.999 \pm 0.000$ & $208.4 \pm 1.780$ & $1.000 \pm 0.000$ & - & 
      $1.000 \pm 0.000$ & $998.9 \pm 0.224$ & $1.000 \pm 0.000$ & - \\
      \textsc{CRSVP} &
      $0.900 \pm 0.003$ & $44.69 \pm 1.252$ & $1.000 \pm 0.000$ & $0.100 \pm 0.007$ & 
      $0.900 \pm 0.002$ & $520.9 \pm 4.745$ & $1.000 \pm 0.000$ & $0.340 \pm 0.004$\\
      \midrule
      \textsc{NCRSVP-1} & 
       $0.999 \pm 0.000$ & $215.5 \pm 1.885$ & $1.000 \pm 0.000$ & - & 
       $1.000 \pm 0.000$ & $998.9 \pm 0.176$ & $1.000 \pm 0.000$ & - \\
      \textsc{CRSVP-1} &
       $0.901 \pm 0.003$ & $44.24 \pm 1.297$ & $1.000 \pm 0.000$ & $0.139 \pm 0.004$ &
       $0.900 \pm 0.003$ & $534.5 \pm 4.492$ & $1.000 \pm 0.000$ & $0.356 \pm 0.005$ \\
      \midrule
      \textsc{NCRSVP-2} &
      $0.999 \pm 0.000$ & $179.0 \pm 2.035$ & $1.242 \pm 0.004$ & - &
       $1.000 \pm 0.000$ & $997.0 \pm 0.443$ & $1.002 \pm 0.001$ & - \\
      \textsc{CRSVP-2} &
      $0.901 \pm 0.003$ & $27.63 \pm 0.898$ & $1.270 \pm 0.006$ & $0.073 \pm 0.007$ & 
       $0.898 \pm 0.003$ & $442.1 \pm 4.337$ & $1.290 \pm 0.004$ & $0.295 \pm 0.004$ \\
      \midrule
      \textsc{NCRSVP-3} &
      $1.000 \pm 0.000$ & $163.6 \pm 2.070$ & $1.632 \pm 0.005$ & - & 
      $1.000 \pm 0.000$ & $995.3 \pm 0.492$ & $1.006 \pm 0.001$ & - \\
      \textsc{CRSVP-3} &
      $0.901 \pm 0.003$ & $20.30 \pm 0.830$ & $1.498 \pm 0.009$ & $0.045 \pm 0.009$ & 
      $0.899 \pm 0.004$ & $389.7 \pm 5.898$ & $1.632 \pm 0.010$ & $0.262 \pm 0.004$ \\
      \bottomrule
  \end{tabular}
  }
\end{table*}

\begin{figure}[t]
  \renewcommand{\thefigure}{C\arabic{figure}}
  \centering
\includegraphics[width=0.9\textwidth]{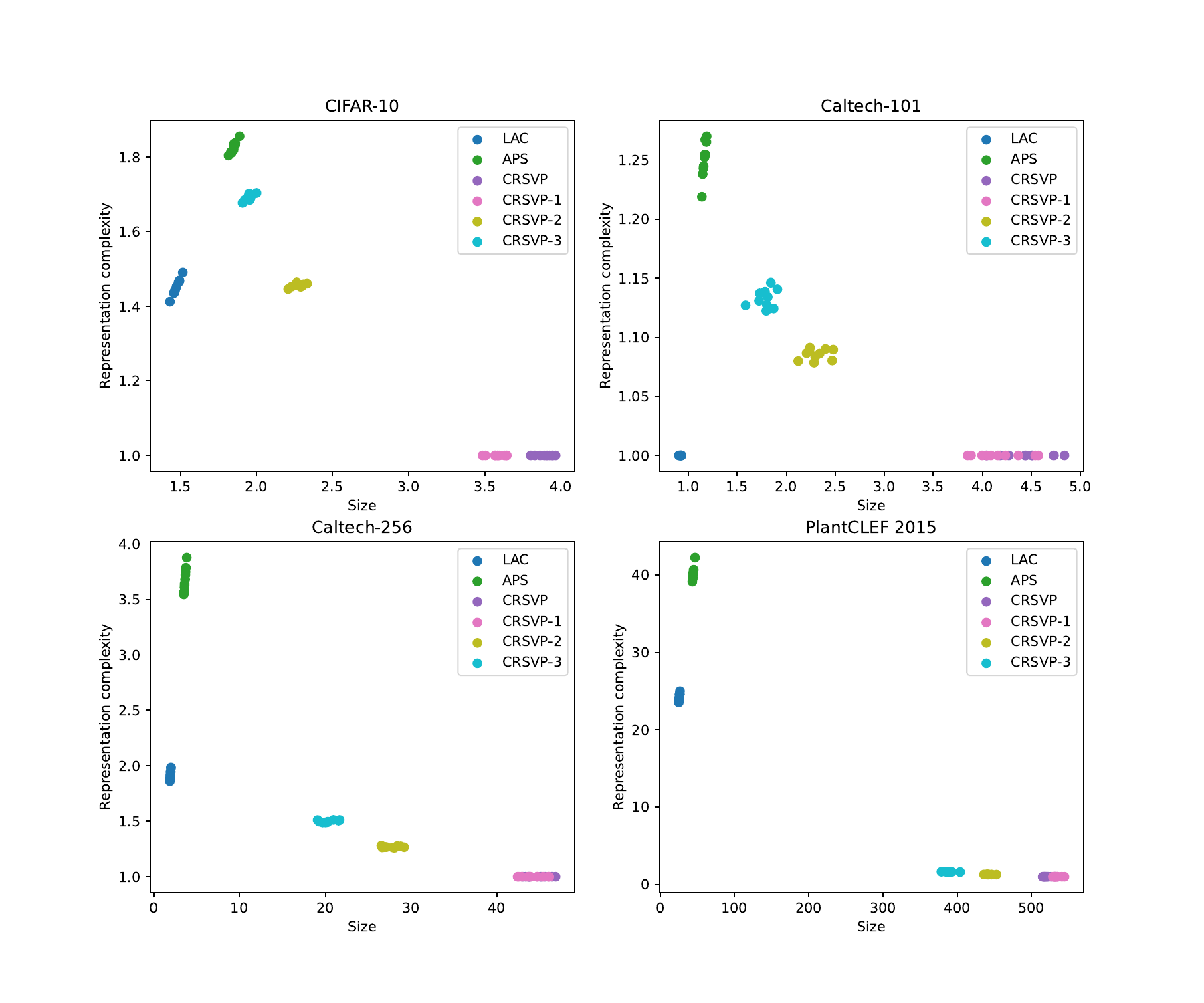}
\caption{Representation complexity versus set size for the following unrestricted set-valued predictors: LAC and APS, and restricted set-valued predictors: CRSVP and CRSVP-$r$, for all datasets. The confidence level is set to 90\%, and calibration and test sets are resampled 10 times.}
\label{fig:supp:extresults}
\end{figure}

\end{document}